\documentclass{article}

\PassOptionsToPackage{numbers, compress}{natbib}
\usepackage[preprint]{nips_2018}

\usepackage[utf8]{inputenc} 
\usepackage[T1]{fontenc}    
\usepackage[bookmarks=false,pdfstartview={FitH}]{hyperref}       
\usepackage{url}            
\usepackage{booktabs}       
\usepackage{amsfonts}       
\usepackage{nicefrac}       
\usepackage{microtype}      
\usepackage{graphicx}
\usepackage{times} 
\usepackage{amsmath} 
\usepackage{amsthm} 
\usepackage{amssymb}
\usepackage{algorithmic}  
\usepackage{algorithm} 
\usepackage{epsfig} 
\usepackage{epstopdf}
\usepackage{multirow} 
\usepackage{subfigure} 
\usepackage{caption} 
\usepackage{lipsum} 
\usepackage{xcolor} 
\usepackage{bbm}
\usepackage{enumitem}

\newtheorem{theorem}{\bf{Theorem}}

\title{Maximum Causal Tsallis Entropy Imitation Learning}

\author{
   Kyungjae Lee\\
  Seoul National University\\
  \texttt{kyungjae.lee@cpslab.snu.ac.kr} \\
 \And
  Sungjoon Choi\\
  Kakao Brain\\
  \texttt{sam.choi@kakaobrain.com} \\
 \And
   Songhwai Oh\\
  Seoul National University\\
  \texttt{songhwai.oh@cpslab.snu.ac.kr} \\
}

\begin{document}

\maketitle

\begin{abstract}
In this paper, we propose a novel maximum causal Tsallis entropy
(MCTE) framework for imitation learning which can efficiently learn a
sparse multi-modal policy distribution from demonstrations. 
We provide the full mathematical analysis of the proposed framework.
First, the optimal solution of an MCTE problem is shown to be a sparsemax
distribution, whose supporting set can be adjusted. 
The proposed method has advantages over
a softmax distribution in that it can exclude unnecessary actions by
assigning zero probability.
Second, we prove that an MCTE problem is equivalent to robust Bayes estimation
 in the sense of the Brier score.
Third, we propose a maximum causal Tsallis entropy imitation learning
(MCTEIL) algorithm with a sparse mixture density network (sparse MDN)
by modeling mixture weights using a sparsemax distribution. 
In particular, we show that the causal Tsallis entropy of an MDN
encourages exploration and efficient mixture utilization while Boltzmann Gibbs entropy is less effective.
We validate the proposed method in two simulation studies and MCTEIL
outperforms existing imitation learning methods in terms of average
returns and learning multi-modal policies. 
\end{abstract}

\section{Introduction}
In this paper, we focus on the problem of imitating demonstrations of
an expert who behaves non-deterministically depending on the situation.
In imitation learning, it is often assumed that the expert's policy is
deterministic. 
However, there are instances, especially for complex tasks, where 
multiple action sequences perform the same task equally well. 
We can model such nondeterministic behavior of an expert using a
stochastic policy.
For example, expert drivers normally show consistent
behaviors such as keeping lane or keeping the distance from a frontal
car, but sometimes they show different actions for the same
situation, such as overtaking a car and turning left or right at an
intersection, as suggested in \cite{ziebart2008maximum}.
Furthermore, learning multiple optimal action sequences to perform a
task is desirable in terms of robustness since an agent can easily
recover from failure due to unexpected events \cite{Haarnoja2017, lee2018sparse}. 
In addition, a stochastic policy promotes exploration and stability
during learning \cite{Heess2012, Haarnoja2017, vamplew2017softmax}.
Hence, modeling experts' stochasticity can be a key factor in imitation
learning.

To this end, we propose a novel maximum causal Tsallis entropy (MCTE)
framework for imitation learning, which can learn from a uni-modal to 
multi-modal policy distribution by adjusting its supporting set. 
We first show that the optimal policy under the MCTE framework follows
a \textit{sparsemax} distribution \cite{martins2016softmax}, which
has an adaptable supporting set in a discrete action space.
Traditionally, the maximum causal entropy (MCE) framework
\cite{ziebart2008maximum,bloem2014infinite} has been proposed to model
stochastic behavior in demonstrations, where the optimal policy
follows a softmax distribution. 
However, it often assigns non-negligible probability mass to
non-expert actions when the number of actions increases
\cite{lee2018sparse, nachum2018path}. 
On the contrary, as the optimal policy of the proposed method can
adjust its supporting set, it can model various expert's behavior from
a uni-modal distribution to a multi-modal distribution. 

To apply the MCTE framework to a complex and model-free problem, we
propose a maximum causal Tsallis entropy imitation learning (MCTEIL)
with a sparse mixture density network (sparse MDN) whose mixture
weights are modeled as a sparsemax distribution. 
By modeling expert's behavior using a sparse MDN, MCTEIL can learn
varying stochasticity depending on the state in a continuous action
space.  
Furthermore, we show that the MCTEIL algorithm can be obtained by
extending the MCTE framework to the generative adversarial setting,
similarly to generative adversarial imitation learning (GAIL) by Ho
and Ermon \cite{ho2016generative}, which is based on the MCE
framework. 
The main benefit of the generative adversarial setting is that the
resulting policy distribution is more robust than that of a supervised
learning method since it can learn recovery behaviors from less
demonstrated regions to demonstrated regions by exploring the
state-action space during training. 
Interestingly, we also show that the Tsallis entropy of a sparse MDN
has an analytic form and is proportional to the distance between
mixture means.  
Hence, maximizing the Tsallis entropy of a sparse MDN encourages
exploration by providing bonus rewards to wide-spread mixture means and
penalizing collapsed mixture means, while the causal entropy
\cite{ziebart2008maximum} of an MDN is less effective in terms of
preventing the collapse of mixture means since there is no analytical
form and its approximation is used in practice instead.
Consequently, maximizing the Tsallis entropy of a sparse MDN has a
clear benefit over the causal entropy in terms of exploration and
mixture utilization. 

To validate the effectiveness of the proposed method, we conduct two
simulation studies. 
In the first simulation study, we verify that MCTEIL with a sparse MDN can
successfully learn multi-modal behaviors from expert's demonstrations. 
A sparse MDN efficiently learns a multi-modal policy without 
performance loss, while a single Gaussian and a softmax-based MDN
suffer from performance loss.
The second simulation study is conducted using four continuous control
problems in MuJoCo \cite{todorov2012mujoco}. 
MCTEIL outperforms existing methods in terms of the average cumulative return.
In particular, MCTEIL shows the best performance for the
\textit{reacher} problem with a smaller number of demonstrations while
GAIL often fails to learn the task.

\section{Background} \label{BG}

\paragraph{Markov Decision Processes}
Markov decision processes (MDPs) are a well-known mathematical
framework for a sequential decision making problem. 
A general MDP is defined as a tuple $\{\mathcal{S}, \mathcal{F}, \mathcal{A}, \phi, \Pi, d, T, \gamma, \mathbf{r} \}$, 
where $\mathcal{S}$ is the state space, 
$\mathcal{F}$ is the corresponding feature space, 
$\mathcal{A}$ is the action space,
$\phi$ is a feature map from $\mathcal{S} \times \mathcal{A}$ to $\mathcal{F}$,
$\Pi$ is a set of stochastic policies, i.e., $ \Pi=\{\pi ~|~ \forall s \in \mathcal{S},\; a \in \mathcal{A},\; \pi(a|s)\ge 0 \;\text{and} \; \sum_{a'}\pi(a'|s)=1\}$,
$d(s)$ is the initial state distribution,
$T(s'|s,a)$ is the transition probability from $s\in\mathcal{S}$ to $ s'\in\mathcal{S}$ by taking $a\in\mathcal{A}$,
$\gamma \in (0,1)$ is a discount factor, and
$\mathbf{r}$ is the reward function from a state-action pair to a real value.
In general, the goal of an MDP is to find the optimal policy distribution $\pi^{*} \in \Pi$ which maximizes the expected discount sum of rewards, i.e., $\mathbb{E}_{\pi}\left[\mathbf{r}(s,a)\right] \triangleq \mathbb{E}\left[\sum_{t=0}^{\infty}\mathbf{r}(s_t,a_t)\middle|\pi,d\right]$.
Note that, for any function $f(s,a)$, $\mathbb{E}\left[\sum_{t=0}^{\infty}f(s_t,a_t)\middle|\pi,d\right]$ will be denoted as $\mathbb{E}_{\pi}\left[f(s,a)\right]$.

\paragraph{Maximum Causal Entropy Inverse Reinforcement Learning}
Zeibart et al. \cite{ziebart2008maximum} proposed the maximum causal entropy framework, 
which is also known as maximum entropy inverse reinforcement learning (MaxEnt IRL).
MaxEnt IRL maximizes the causal entropy of a policy distribution
while the feature expectation of the optimized policy distribution is
matched with that of expert's policy. 
The maximum causal entropy framework is defined as follows:
\begin{eqnarray}
\begin{aligned}\label{eqn:maxent}
& \underset{\pi \in \Pi}{\text{maximize}}
& & \alpha H(\pi) \\
& \text{subject to}
& & \mathbb{E}_{\pi}\left[\phi(s,a)\right]  = \mathbb{E}_{\pi_{E}}\left[\phi(s,a)\right],
\end{aligned}
\end{eqnarray}
where $H(\pi) \triangleq  \mathbb{E}_{\pi}\left[-\log(\pi(a|s))\right]$ 
is the causal entropy of policy $\pi$, $\alpha$ is a scale parameter,
$\pi_{E}$ is the policy distribution of the expert.
Maximum casual entropy estimation finds the most uniformly distributed
policy satisfying feature matching constraints. 
The feature expectation of the expert policy is used as a statistic to
represent the behavior of an expert and is approximated from expert's
demonstrations $\mathcal{D} = \{\zeta_0,\cdots,\zeta_N\}$, where $N$
is the number of demonstrations and $\zeta_i$ is a sequence of state
and action pairs whose length is $T$, 
i.e., $\zeta_i = \{(s_0,a_0),\cdots,(s_T,a_T)\}$. 
In \cite{ziebart2010MPAs}, it is shown that the optimal solution of
(\ref{eqn:maxent}) is a softmax distribution.

\paragraph{Generative Adversarial Imitation Learning}
In \cite{ho2016generative}, Ho and Ermon have extended
(\ref{eqn:maxent}) to a unified framework for IRL by adding a
reward regularization as follows: 
\begin{eqnarray}
\begin{aligned}\label{eqn:unif_irl}
& \underset{c}{\text{max}} \; \underset{\pi \in \Pi}{\text{min}}
& & -\alpha H(\pi)+\mathbb{E}_{\pi}\left[c(s,a)\right] - \mathbb{E}_{\pi_{E}}\left[c(s,a)\right] - \psi(c),\\
\end{aligned}
\end{eqnarray}
where $c$ is a cost function and $\psi$ is a convex regularization for
cost $c$.
As shown in \cite{ho2016generative}, many existing IRL methods can be
interpreted with this framework, such as 
MaxEnt IRL \cite{ziebart2008maximum}, 
apprenticeship learning \cite{abbeel2004apprenticeship}, and 
multiplicative weights apprenticeship learning \cite{syed2008game}. 
Existing IRL methods based on (\ref{eqn:unif_irl}) often require to 
solve the inner minimization over $\pi$ for fixed $c$ in order to
compute the gradient of $c$. 
In \cite{ziebart2010MPAs}, Ziebart showed that the inner minimization
is equivalent to a soft Markov decision process (soft MDP) under the reward $-c$
and proposed soft value iteration to solve the soft MDP.
However, solving a soft MDP every iteration is often intractable for
problems with large state and action spaces and also requires
the transition probability which is not accessible in many cases. 
To address this issue, the generative adversarial imitation learning
(GAIL) framework is proposed in \cite{ho2016generative} to avoid
solving the soft MDP problem directly.
The unified imitation learning problem (\ref{eqn:unif_irl}) can be
converted into the GAIL framework as follows:
\begin{eqnarray}
\begin{aligned}
& \underset{\pi  \in \Pi}{\text{min}} \; \underset{\mathbf{D}}{\text{max}} && \mathbb{E}_{\pi}\left[\log(\mathbf{D}(s,a))\right] + \mathbb{E}_{\pi_{E}}\left[\log(1-\mathbf{D}(s,a))\right] - \alpha H(\pi),
\end{aligned}
\end{eqnarray}
where $\mathbf{D} \in (0,1)^{|\mathcal{S}||\mathcal{A}|}$ indicates a
discriminator, which returns the probability that a given
demonstration is from a learner, i.e., $1$ for learner's demonstrations
and $0$ for expert's demonstrations.
Notice that we can interpret $\log(\mathbf{D})$ as cost $c$
(or reward of $-c$).

Since existing IRL methods, including GAIL, are often based
on the maximum causal entropy, they model the expert's policy using a
softmax distribution, which can assign non-zero probability to non-expert
actions in a discrete action space. 
Furthermore, in a continuous action space, expert's behavior is often
modeled using a uni-modal Gaussian distribution, which is not proper to
model multi-modal behaviors. 
To handle these issues, we propose a sparsemax distribution as the
policy of an expert and provide a natural extension to handle a
continuous action space using a mixture density network with sparsemax
weight selection. 

\paragraph{Sparse Markov Decision Processes} 
In \cite{lee2018sparse}, a sparse Markov decision process (sparse
MDP) is proposed by utilizing the causal sparse Tsallis entropy 
$W(\pi) \triangleq \frac{1}{2}\mathbb{E}_{\pi}\left[1-\pi(a|s)\right]$
to the expected discounted rewards sum, i.e.,
$\mathbb{E}_{\pi}\left[\mathbf{r}(s,a)\right] + \alpha W(\pi)$. 
Note that $W(\pi)$ is an extension of a special case of the
generalized Tsallis entropy, i.e., 
$S_{k,q}(p) = \frac{k}{q-1} \left(1- \sum_{i} p_{i}^{q} \right)$, 
for $k=\frac{1}{2}, q=2$, to sequential random variables. 
It is shown that that the optimal policy of a sparse MDP is a sparse
and multi-modal policy distribution \cite{lee2018sparse}. 
Furthermore, sparse Bellman optimality conditions were derived as follows:
\begin{eqnarray} \label{eqn:spbellman}
\small
\begin{aligned}
Q(s,a) &\triangleq r(s,a) + \gamma \sum_{s'}V(s')T(s'|s,a),\; \pi(a|s) = \max\left(\frac{Q(s,a)}{\alpha} - \tau\left(\frac{Q(s,\cdot)}{\alpha}\right),0\right),\\
V(s) &= \alpha\left[\frac{1}{2}\sum_{a\in S(s)}\left(\left(\frac{Q(s,a)}{\alpha}\right)^{2} - \tau \left(\frac{Q(s,\cdot)}{\alpha}\right)^{2}\right) + \frac{1}{2}\right],
\end{aligned}
\end{eqnarray}
where $\tau\left(\frac{Q(s,\cdot)}{\alpha}\right) = \frac{\sum_{a\in S(s)}\frac{Q(s,a)}{\alpha} - 1}{K_s}$, $S(s)$ is a set of actions 
satisfying $1 + i\frac{Q(s,a_{(i)})}{\alpha}>\sum_{j=1}^{i}\frac{Q(s,a_{(j)})}{\alpha}$
with $a_{(i)}$  indicating the action with the $i$th largest state-action value $Q(s,a)$,
and $K_s$ is the cardinality of $S(s)$.
In \cite{lee2018sparse}, 
a sparsemax policy shows better performance compared to a softmax
policy since it assigns zero probability to non-optimal actions whose
state-action value is below the threshold $\tau$. 
In this paper, we utilize this property in imitation learning by
modeling expert's behavior using a sparsemax distribution. 
In Section \ref{sec:mcte}, we show that the optimal solution of an MCTE
problem also has a sparsemax distribution and, hence, the optimality
condition of sparse MDPs is closely related to that of MCTE problems. 

\section{Principle of Maximum Causal Tsallis Entropy} \label{sec:mcte}

In this section, we formulate maximum causal Tsallis entropy imitation
learning (MCTEIL) and show that MCTE induces a sparse and multi-modal
distribution which has an adaptable supporting set. 
The problem of maximizing the causal Tsallis entropy $W(\pi)$ can be
formulated as follows: 
\begin{eqnarray}\label{eqn:maximum_cste}
\begin{aligned}
& \underset{\pi \in \Pi}{\text{maximize}}
& & \alpha W(\pi)\\
& \text{subject to}
& & \mathbb{E}_{\pi}\left[\phi(s,a)\right]  = \mathbb{E}_{\pi_{E}}\left[\phi(s,a)\right].
\end{aligned}
\end{eqnarray}
In order to derive optimality conditions, we will first change the optimization variable from a policy
distribution to a state-action visitation measure. 
Then, we prove that the MCTE problem is concave with respect to the
visitation measure.
The necessary and sufficient conditions for an optimal solution are
derived from the Karush-Kuhn-Tucker (KKT) conditions using the strong
duality and the optimal policy is shown to be a sparsemax distribution. 
Furthermore, we also provide an interesting interpretation of the MCTE
framework as robust Bayes estimation in terms of the Brier score.
Hence, the proposed method can be viewed as maximization of the worst
case performance in the sense of the Brier score \cite{brier1950verification}.

We can change the optimization variable from a policy distribution
to a state-action visitation measure based on the following theorem.
\begin{theorem}[Theorem 2 of Syed et al. \cite{syed2008apprenticeship}]\label{thm:otocrrsp}
Let $\mathbf{M}$ be a set of state-action visitation measures, i.e., $\mathbf{M} \triangleq\{\rho| \forall s,\;a,\; \rho(s,a)\ge 0,\; \sum_{a}\rho(s,a) = d(s) + \sum_{s',a'} T(s|s',a')\rho(s',a')\}$.
If $\rho \in \mathbf{M}$, then it is a state-action visitation measure
for $\pi_{\rho}(a|s) \triangleq \frac{\rho(s,a)}{\sum_{a}\rho(s,a)}$, and 
$\pi_{\rho}$ is the unique policy whose state-action visitation
measure is $\rho$.
\end{theorem}
\begin{proof}
The proof can be found in \cite{syed2008apprenticeship}.
\end{proof}
Theorem \ref{thm:otocrrsp} guarantees the one-to-one correspondence
between a policy distribution and state-action visitation measure. 
Then, the objective function $W(\pi)$ is converted into the function of $\rho$ as follows.
\begin{theorem}\label{thm:obj_chg}
Let $\bar{W}(\rho) = \frac{1}{2}\sum_{s,a} \rho(s,a) \left(1-\frac{\rho(s,a)}{\sum_{a'}\rho(s,a')}\right)$. 
Then, for any stationary policy $\pi \in \Pi$ and any state-action visitation
measure $\rho \in \mathbf{M}$, 
$W(\pi)=\bar{W}(\rho_{\pi})$ and $\bar{W}(\rho) = W(\pi_{\rho})$ hold.
\end{theorem} 
The proof is provided in the supplementary material.
Theorem \ref{thm:obj_chg} tells us that if $\bar{W}(\rho)$ has the maximum at $\rho^{*}$, then $W(\pi)$ also has the maximum at $\pi_{\rho^{*}}$.
Based on Theorem \ref{thm:otocrrsp} and \ref{thm:obj_chg}, we can freely convert the problem (\ref{eqn:maximum_cste}) into 
\begin{eqnarray}\label{eqn:maximum_cste_vis}
\begin{aligned}
& \underset{\rho \in \mathbf{M}}{\text{maximize}}
& & \alpha \bar{W}(\rho)\\
& \text{subject to}
&&\sum_{s,a}\rho(s,a)\phi(s,a)  = \sum_{s,a}\rho_{E}(s,a)\phi(s,a),
\end{aligned}
\end{eqnarray}
where $\rho_{E}$ is the state-action visitation measure corresponding to $\pi_{E}$.

\subsection{Optimality Condition of Maximum Causal Tsallis Entropy}

We show that the optimal policy of the problem
(\ref{eqn:maximum_cste_vis}) is a sparsemax distribution using the KKT
conditions. 
In order to use the KKT conditions, we first show that the MCTE problem is concave.
\begin{theorem}\label{thm:obj_concave}
$\bar{W}(\rho)$ is strictly concave with respect to $\rho \in \mathbf{M}$.
\end{theorem}
The proof of Theorem \ref{thm:obj_concave} is provided in the supplementary material.
Since all constraints are linear and the objective function is concave,
(\ref{eqn:maximum_cste_vis}) is a concave problem and, hence, strong duality holds. 
The dual problem is defined as follows:
\begin{eqnarray} \label{eqn:dual_max_cste_vis}
\begin{aligned}
& \underset{\theta,c,\lambda}{\text{max}} \;\; \underset{\rho}{\text{min}}
& & L_{W}(\theta, c, \lambda, \rho)\\
& \text{subject to}
& & \forall \, s,a  \;\; \lambda_{sa}\ge0,
\end{aligned}
\end{eqnarray}
where $L_{W}(\theta, c, \lambda, \rho) = -\alpha \bar{W}(\rho)-\sum_{s,a}\rho(s,a)\theta^{\intercal} \phi(s,a)+ \sum_{s,a}\rho_{E}(s,a)\theta^{\intercal} \phi(s,a) - \sum_{s,a} \lambda_{sa} \rho(s,a)+ \sum_{s} c_s \left(\sum_{a}\rho(s,a) - d(s) - \gamma \sum_{s',a'} T(s|s',a')\rho(s',a')\right)$
and $\theta$, $c$, and $\lambda$ are Lagrangian multipliers and the constraints come from $\mathbf{M}$.
Then, the optimal solution of primal and dual variables necessarily and sufficiently satisfy the KKT conditions.
\begin{theorem}\label{thm:necessarily_cnd}
The optimal solution of (\ref{eqn:maximum_cste_vis}) sufficiently and
necessarily satisfies the following conditions:
\begin{eqnarray*}
\begin{aligned}
q_{sa} \triangleq \theta^{\intercal}\phi(s,a) + \gamma \sum_{s'}c_{s'}T(s'|s,a),& 
\;c_{s} = \alpha\left[\frac{1}{2}\sum_{a\in S(s)}\left(\left(\frac{q_{sa}}{\alpha}\right)^{2} - \tau \left(\frac{q_{s}}{\alpha}\right)^{2}\right) + \frac{1}{2}\right],\\
\textrm{and} \qquad \pi_{\rho}(a|s) = & \max\left(\frac{q_{sa}}{\alpha} - \tau
\left(\frac{q_{s}}{\alpha}\right),0\right),
\end{aligned}
\end{eqnarray*}
where $\pi_{\rho}(a|s) = \frac{\rho(s,a)}{\sum_{a}\rho(s,a)}$,
$q_{sa}$ is an auxiliary variable, and 
$q_s = [q_{s a_1} \cdots q_{s a_{|\mathcal{A}|}}]^{\intercal}$.
\end{theorem}
The optimality conditions of the problem (\ref{eqn:maximum_cste_vis})
tell us that the optimal policy is a sparsemax
distribution which assigns zero probability to an action whose
auxiliary variable $q_{sa}$ is below the threshold $\tau$, which
determines a supporting set. 
If expert's policy is multi-modal at state $s$, the resulting
$\pi_{\rho}(\cdot|s)$ becomes multi-modal and induces
a multi-modal distribution with a large supporting set. 
Otherwise, the resulting policy has a sparse and smaller supporting set.
Therefore, a sparsemax policy has advantages over a softmax policy for
modeling sparse and multi-modal behaviors of an expert whose
supporting set varies according to the state. 

Furthermore, we also discover an interesting connection between the
optimality condition of an MCTE problem and the sparse Bellman
optimality condition (\ref{eqn:spbellman}). 
Since the optimality condition
is equivalent to the sparse Bellman optimality equation
\cite{lee2018sparse}, we can compute the optimal policy
and Lagrangian multiplier $c_{s}$ by solving a sparse MDP under the
reward function $\mathbf{r}(s,a) = {\theta^*}^{\intercal}\phi(s,a)$,
where $\theta^*$ is the optimal dual variable.
In addition, $c_{s}$ and $q_{sa}$ can be viewed as a state value and
state-action value for the reward ${\theta^*}^{\intercal}\phi(s,a)$,
respectively.

\subsection{Interpretation as Robust Bayes}

In this section, we provide an interesting interpretation about the
MCTE framework. 
In general, maximum entropy estimation can be viewed as a minimax game
between two players.
One player is called a decision maker and the other player is called
the nature, where the nature assigns a distribution to maximize the
decision maker's misprediction while the decision maker tries to
minimize it \cite{grunwald2004game}. 
The same interpretation can be applied to the MCTE framework.
We show that the proposed MCTE problem is equivalent to a minimax game
with the Brier score \cite{brier1950verification}.
\begin{theorem}\label{thm:maximinprob}
The maximum causal Tsallis entropy distribution minimizes the worst
case prediction Brier score,
\begin{eqnarray}\label{eqn:minmaxprob}
\begin{aligned}
\underset{\pi \in \Pi}{\min}\;\underset{\tilde{\pi} \in \Pi}{\max}\;\; \mathbb{E}_{\tilde{\pi}}\left[\sum_{a'}\frac{1}{2}\left(\mathbbm{1}_{\{a'=a\}} - \pi(a|s)\right)^2\right]
\;\;\textnormal{subject to}\;\;\mathbb{E}_{\pi}\left[\phi(s,a)\right]  = \mathbb{E}_{\pi_{E}}\left[\phi(s,a)\right]
\end{aligned}
\end{eqnarray}
where $\sum_{a'}\frac{1}{2}\left(\mathbbm{1}_{\{a'=a\}} -
\pi(a|s)\right)^2$ is the Brier score.
\end{theorem}
Note that minimizing the Brier score minimizes the misprediction ratio
while we call it a score here.
Theorem \ref{thm:maximinprob} is a straightforward extension of the
robust Bayes results in \cite{grunwald2004game} to sequential
decision problems.  
This theorem tells us that the MCTE problem can be viewed as a minimax
game between a sequential decision maker $\pi$ and the nature
$\tilde{\pi}$ based on the Brier score. 
In this regards, the resulting estimator can be interpreted as the
best decision maker against the worst that the nature can offer.

\section{Maximum Causal Tsallis Entropy Imitation Learning}

\begin{algorithm}[t!]
\caption{Maximum Causal Tsallis Entropy Imitation Learning}
\begin{algorithmic}[1] \label{sgail}
\small
\STATE Expert's demonstrations $\mathcal{D}$ are given
\STATE Initialize policy and discriminator parameters $\nu, \omega$
\WHILE {until convergence}
\STATE Sample trajectories $\{\zeta\}$ from $\pi_{\nu}$
\STATE Update $\omega$ with the gradient of $\sum_{\{\zeta\}}\log(\mathbf{D_{\omega}}(s,a)) + \sum_{\mathcal{D}}\log(1-\mathbf{D_{\omega}}(s,a))$.
\STATE Update $\nu$ using a policy optimization method with reward function $\mathbbm{E}_{\pi_{\nu}} \left[\log(\mathbf{D_{\omega}}(s,a))\right]+\alpha W(\pi_{\nu})$
\ENDWHILE
\end{algorithmic}
\end{algorithm}

In this section, we propose a maximum causal Tsallis entropy imitation
learning (MCTEIL) algorithm to solve a model-free IL problem in a
continuous action space. 
In many real-world problems, state and action spaces are often
continuous and transition probability of a world cannot be
accessed. 
To apply the MCTE framework for a continuous space and model-free case,
we follow the extension of GAIL
\cite{ho2016generative}, which trains a policy and reward
alternatively, instead of solving RL at every iteration.  
We extend the MCTE framework to a more general case with reward
regularization and it is formulated by replacing the causal entropy
$H(\pi)$ in the problem (\ref{eqn:unif_irl}) with the causal Tsallis
entropy $W(\pi)$ as follows: 
\begin{eqnarray}
{\small
\begin{aligned}\label{eqn:unif_sirl}
& \underset{\theta}{\text{max}} \; \underset{\pi \in \Pi}{\text{min}}
& & -\alpha W(\pi)-\mathbb{E}_{\pi}\left[\theta^{\intercal}\phi(s,a)\right] + \mathbb{E}_{\pi_{E}}\left[\theta^{\intercal}\phi(s,a)\right] - \psi(\theta).
\end{aligned}
}
\end{eqnarray}
Similarly to \cite{ho2016generative}, we convert the problem (\ref{eqn:unif_sirl}) into the generative adversarial setting as follows.
\begin{theorem}\label{thm:gail_setting}
The maximum causal sparse Tsallis entropy problem (\ref{eqn:unif_sirl}) is equivalent to the following problem:
\begin{eqnarray*}
\begin{aligned}
& \underset{\pi \in \Pi}{\textnormal{min}} && \psi^{*}\left(\mathbb{E}_{\pi}\left[\phi(s,a)\right] -  \mathbb{E}_{\pi_{E}}\left[\phi(s,a)\right]\right) - \alpha W(\pi),
\end{aligned}
\end{eqnarray*}
where $\psi^{*}(x) = \sup_{y} \{y^{\intercal} x - \psi(y)\}$.
\end{theorem}
The proof is detailed in the supplementary material.
The proof of Theorem \ref{thm:gail_setting} depends on the fact that
the objective function of (\ref{eqn:unif_sirl}) is concave with
respect to $\rho$ and is convex with respect to $\theta$. 
Hence, we first switch the optimization variables from $\pi$ to $\rho$
and, using the minimax theorem \cite{millar1983minimax}, the
maximization and minimization are interchangeable and the generative
adversarial setting is derived. 
Similarly to \cite{ho2016generative}, Theorem \ref{thm:gail_setting}
says that a MCTE problem can be interpreted as minimization of the
distance between expert's feature expectation and training policy's
feature expectation, where $\psi^{*}(x_1 - x_2)$ is a proper distance
function since $\psi(x)$ is a convex function.
Let $e_{sa} \in \mathbb{R}^{|\mathcal{S}||\mathcal{A}|}$ be a feature
indicator vector, such that the $sa$th element is one and zero elsewhere.
If we set $\psi$ to 
$\psi_{GA}(\theta) \triangleq \mathbb{E}_{\pi_{E}}[g(\theta^{\intercal}e_{sa})]$, 
where $g(x) = -x-\log(1-e^{x})$ for $x < 0$ and $g(x) = \infty$ for $x \ge 0$,
we can convert the MCTE problem into the following generative adversarial setting: 
\begin{eqnarray}\label{eqn:sgail_prob}
\begin{aligned}
& \underset{\pi  \in \Pi}{\text{min}} \; \underset{\mathbf{D}}{\text{max}} && \mathbb{E}_{\pi}\left[\log(\mathbf{D}(s,a))\right] + \mathbb{E}_{\pi_{E}}\left[\log(1-\mathbf{D}(s,a))\right] - \alpha W(\pi),
\end{aligned}
\end{eqnarray}
where $\mathbf{D}$ is a discriminator.
The problem (\ref{eqn:sgail_prob}) can be solved by MCTEIL which
consists of three steps. 
First, trajectories are sampled from the training policy $\pi_{\nu}$
and discriminator $\mathbf{D}_{\omega}$ is updated to distinguish
whether the trajectories are generated by $\pi_{\nu}$ or $\pi_{E}$. 
Finally, the training policy $\pi_{\nu}$ is updated with a policy
optimization method under the sum of rewards
$\mathbb{E}_{\pi}\left[-\log(\mathbf{D}_{\omega}(s,a))\right]$ with
a causal Tsallis entropy bonus $\alpha W(\pi_{\nu})$. 
The algorithm is summarized in Algorithm \ref{sgail}.

\paragraph{Sparse Mixture Density Network}

We further employ a novel mixture density network (MDN) with sparsemax
weight selection, which can model sparse and multi-modal behavior of
an expert, which is called a sparse MDN. 
In many imitation learning algorithms, a Gaussian network is often
employed to model expert's policy in a continuous action space.
However, a Gaussian distribution is inappropriate to model the
multi-modality of an expert since it has a single mode.
An MDN is more suitable for modeling a multi-modal distribution.
In particular, a sparse MDN is a proper extension of a sparsemax
distribution for a continuous action space. 
The input of a sparse MDN is state $s$ and the output of a sparse MDN
is components of $K$ mixtures of Gaussians: mixture weights $\{w_i\}$,
means $\{\mu_i\}$, and covariance matrices $\{\Sigma_i\}$.
A sparse MDN policy is defined as 
\begin{eqnarray*}
\pi(a|s) = \sum_{i}^{K} w_{i}( s ) \mathcal{N}( a ; \mu_i (s) , \Sigma_i (s) ),
\end{eqnarray*}
where $\mathcal{N}( a ; \mu , \Sigma )$ indicates a multivariate
Gaussian density at point $a$ with mean $\mu$ and covariance
$\Sigma$. 
In our implementation, $w(s)$ is computed as a sparsemax distribution,
while most existing MDN implementations utilize a softmax distribution.
Modeling the expert's policy using an MDN with $K$ mixtures can be
interpreted as separating continuous action space into $K$
representative actions. 
Since we model mixture weights using a sparsemax distribution, the
number of mixtures used to model the expert's policy can vary
depending on the state. 
In this regards, the sparsemax weight selection has an advantage over
the soft weight selection since the former utilizes mixture components
more efficiently as unnecessary components will be assigned with zero
weights. 

\paragraph{Tsallis Entropy of Mixture Density Network} 
An interesting fact is that the causal Tsallis entropy of an MDN has an
analytic form while the Gibbs-Shannon entropy of an MDN is intractable.
\begin{theorem}\label{thm:analytic_mdn}
Let $\pi(a|s) = \sum_{i}^{K} w_{i}( s ) \mathcal{N}( a ; \mu_i (s) , \Sigma_i (s) )$.
Then,
\begin{eqnarray} \label{eqn:tsallis_mdn}
\begin{aligned}
W(\pi) = \frac{1}{2}\sum_{s} \rho_{\pi}(s) \left( 1 - \sum_{i}^{K}\sum_{j}^{K}w_{i}( s )w_{j}( s )\mathcal{N}\left( \mu_{i} (s) ; \mu_{j} (s) , \Sigma_i (s) + \Sigma_j (s) \right)\right).
\end{aligned}
\end{eqnarray}
\end{theorem}
The proof is included in the causal Tsallisrial.
The analytic form of the Tsallis entropy shows that the Tsallis
entropy is proportional to the distance between mixture means. 
Hence, maximizing the Tsallis entropy of a sparse MDN encourages
exploration of diverse directions during the policy optimization step
of MCTEIL. 
In imitation learning, the main benefit of the generative adversarial
setting is that the resulting policy is more robust than that of 
supervised learning since it can learn how to recover from a less
demonstrated region to a demonstrated region by exploring the
state-action space during training. 
Maximum Tsallis entropy of a sparse MDN encourages efficient
exploration by giving bonus rewards when mixture means are spread
out. 
(\ref{eqn:tsallis_mdn}) also has an effect of utilizing
mixtures more efficiently by penalizing for modeling a single mode using
several mixtures. 
Consequently, the Tsallis entropy $W(\pi)$ has clear benefits in terms
of both exploration and mixture utilization. 

\section{Experiments}

To verify the effectiveness of the proposed method, we compare MCTEIL
with several other imitation learning methods.
First, we use behavior cloning (BC) as a baseline.
Second, generative adversarial imitation learning (GAIL) with a single
Gaussian distribution is compared. 
While several variants of GAIL exist
\cite{baram2017end,li2017infogail}, they are all based on the maximum
causal entropy framework and utilize a single Gaussian distribution as
a policy function. 
Hence, we choose GAIL as the representative method.
We also compare a straightforward extension of GAIL for a multi-modal
policy by using a softmax weighted mixture density network (soft MDN)
in order to validate the efficiency of the proposed sparsemax weighted MDN.
In soft GAIL, due to the intractability of the causal entropy of a
mixture of Gaussians, we approximate the entropy term by adding 
$-\alpha\log(\pi(a_t|s_t))$ to $-\log(\mathbf{D}(s_t,a_t))$ 
since $\mathbb{E}_{\pi}\left[-\alpha\log(\mathbf{D}(s,a))\right] + \alpha H(\pi) = \mathbb{E}_{\pi}\left[-\log(\mathbf{D}(s,a)) - \alpha\log(\pi(a|s))\right]$.
The other related imitation learning methods for multi-modal task
learning, such as \cite{hausman2017multi,wang2017robust}, are
excluded from the comparison since they focus on the task level
multi-modality, where the multi-modality of demonstrations comes from
multiple different tasks.
In comparison, the proposed method captures the multi-modality of the
optimal policy for a single task.
We would like to note that our method can be extended to multi-modal
task learning as well.

\subsection{Multi-Goal Environment}


To validate that the proposed method can learn multi-modal behavior of
an expert, we design a simple multi-goal environment with four
attractors and four repulsors, where an agent tries to reach one of
attractors while avoiding all repulsors as shown in Figure
\ref{fig:env}. 
The agent follows the point-mass dynamics and get a positive reward
(resp., a negative reward) when getting closer to an attractor (resp.,
repulsor). 
Intuitively, this problem has multi-modal optimal actions at the center.
We first train the optimal policy using \cite{lee2018sparse} and
generate $300$ demonstrations from the expert's policy. 
For both soft GAIL and MCTEIL, $500$ episodes are sampled at each iteration.
In every iterations, we measure the average return using the underlying rewards and the
reachability which is measured by counting how many goals are reached. 
If the algorithm captures the multi-modality of expert's
demonstrations, then, the resulting policy will show high
reachability. 

The results are shown in Figure \ref{fig:ret} and \ref{fig:reach}.
Since the rewards are multi-modal, it is easy to get a high return if
the algorithm learns only uni-modal behavior. 
Hence, the average returns of soft GAIL and MCTEIL increases similarly.
However, when it comes to the reachability, MCTEIL outperforms soft GAIL
when they use the same number of mixtures. 
In particular, MCTEIL can learn all modes in demonstrations at the end
of learning while soft GAIL suffers from collapsing mixture means. 
This advantage clearly comes from the maximum Tsallis entropy of a sparse MDN 
since the analytic form of the Tsallis entropy directly
penalizes collapsed mixture means while $-\log(\pi(a|s))$ indirectly
prevents modes collapsing in soft GAIL.
Consequently, MCTEIL efficiently utilizes each mixture for wide-spread
exploration. 

\begin{figure}[!t]
\centering
\subfigure[Multi-Goal Environment]{ \label{fig:env}
  \centering
  \includegraphics[width=.32\textwidth]{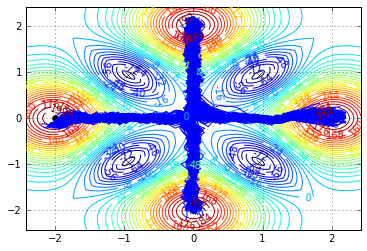}}
\subfigure[Average Return]{ \label{fig:ret}
  \centering
  \includegraphics[width=.32\textwidth]{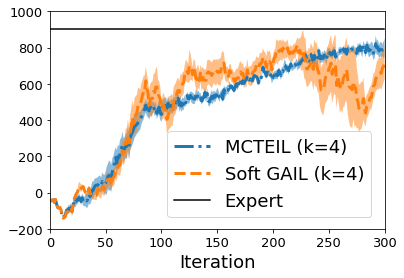}}
  \subfigure[Reachability]{ \label{fig:reach}
  \centering
  \includegraphics[width=.3\textwidth]{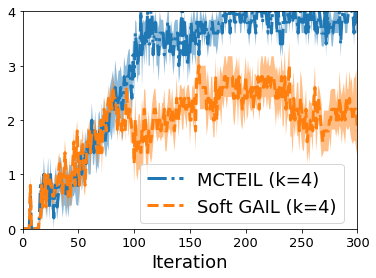}}
\caption{
(a) The environment and multi-modal demonstrations are shown. 
The contour shows the underlying reward map.
(b) The average return of MCTEIL and soft GAIL during training.
(c) The reachability of MCTEIL and soft GAIL during training, where
$k$ is the number of mixtures. 
}
\label{fig:multi_res}
\end{figure}

\subsection{Continuous Control Environment}

We test MCTEIL with a sparse MDN on MuJoCo \cite{todorov2012mujoco},
which is a physics-based simulator, using 
\textit{Halfcheetah, Walker2d, Reacher}, and \textit{Ant}. 
We train the expert policy distribution using trust region policy
optimization (TRPO) \cite{schulman2015trust} under the true reward
function and generate $50$ demonstrations from the expert policy. 
We run algorithms with varying numbers of demonstrations, $4, 11, 18,$
and $25$, and all experiments have been repeated three times with
different random seeds. 
To evaluate the performance of each algorithm, 
we sample $50$ episodes from the trained policy and measure the
average return value using the underlying rewards. 
For methods using an MDN, we use the best number of mixtures using a
brute force search. 

The results are shown in Figure \ref{fig:mujoco_res}.
For three problems, except Walker2d, MCTEIL outperforms the other
methods with respect to the average return as the number of
demonstrations increases. 
For Walker2d, MCTEIL and soft GAIL show similar performance.
Especially, in the reacher problem, we obtain the similar results reported in \cite{ho2016generative}, 
where BC works better than GAIL.
However, our method shows the best performance
for all demonstration counts.
It is observed that the MDN policy tends to show high performance
consistently since MCTEIL and soft GAIL are consistently ranked within the
top two high performing algorithms.
From these results, we can conclude that an MDN policy explores better
than a single Gaussian policy since an MDN can keep searching multiple
directions during training. 
In particular, since the maximum Tsallis entropy makes each mixture
mean explore in different directions and a sparsemax distribution
assigns zero weight to unnecessary mixture components, MCTEIL
efficiently explores and shows better performance compared to soft
GAIL with a soft MDN.
Consequently, we can conclude that MCTEIL outperforms other imitation
learning methods and the causal Tsallis entropy has benefits over the 
causal Gibbs-Shannon entropy as it encourages exploration more
efficiently. 

\begin{figure}[!t]
\centering
\includegraphics[width=\textwidth]{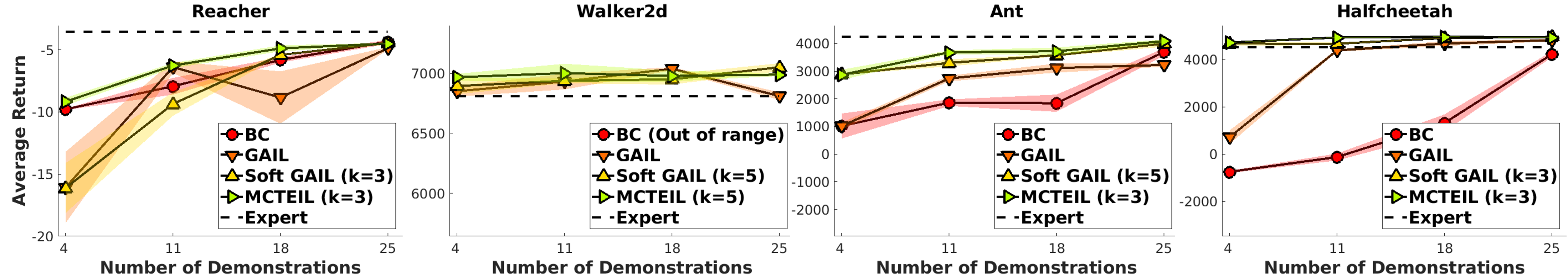}
\caption{
Average returns of trained policies.  
For soft GAIL and MCTEIL, $k$ indicates the number of mixture
and $\alpha$ is an entropy regularization coefficient.  
A dashed line indicates the performance of an expert.
}
\label{fig:mujoco_res}
\end{figure}

\section{Conclusion}
In this paper, we have proposed a novel maximum causal Tsallis entropy
(MCTE) framework, which induces a sparsemax distribution as the optimal
solution. 
We have also provided the full mathematical analysis of the proposed framework,
 including the concavity of the problem, the optimality
condition, and the interpretation as robust Bayes.
We have also developed the maximum causal Tsallis entropy imitation
learning (MCTEIL) algorithm, which can efficiently solve a MCTE
problem in a continuous action space since the Tsallis entropy of a
mixture of Gaussians encourages exploration and efficient mixture
utilization.  
In experiments, we have verified that the proposed method has
advantages over existing methods for learning the multi-modal behavior of
an expert since a sparse MDN can search in diverse directions
efficiently.  
Furthermore, the proposed method has outperformed BC,
GAIL, and GAIL with a soft MDN on the standard IL problems in the
MuJoCo environment. 
From the analysis and experiments, we have shown that the proposed
MCTEIL method is an efficient and principled way to learn the
multi-modal behavior of an expert.

\appendix

\section{Analysis}

\begin{proof}[Proof of \textnormal{Theorem} \ref{thm:obj_chg}]
The proof is simply done by checking two equalities.
First,
\begin{eqnarray*}
\begin{aligned}
W(\pi) &= \frac{1}{2}\mathbb{E}_{\pi}\left[1-\pi(a|s)\right] = \frac{1}{2}\sum_{s,a}\rho_{\pi}(s,a)\left(1-\pi(a|s)\right)\\
&= \frac{1}{2}\sum_{s,a}\rho_{\pi}(s,a)\left(1-\frac{\rho_{\pi}(s,a)}{\sum_{a'}\rho_{\pi}(s,a')}\right)
\end{aligned}
\end{eqnarray*}
and, second,
\begin{eqnarray*}
\begin{aligned}
\bar{W}(\rho) &= \frac{1}{2}\sum_{s,a} \rho(s,a) \left(1-\frac{\rho(s,a)}{\sum_{a'}\rho(s,a')}\right) 
= \frac{1}{2}\sum_{s,a}\rho_{\pi_{\rho}}(s,a)\left(1-\pi_{\rho}(a|s)\right)\\
&= W(\pi_{\rho}).
\end{aligned}
\end{eqnarray*}
\end{proof}

\subsection{Concavity of Maximum causal Tsallis Entropy}

\begin{proof}[Proof of \textnormal{Theorem} \ref{thm:obj_concave}]
Proof of concavity of $\bar{W}(\rho)$ is equivalent to show 
that following inequality is satisfied for all state $s$ and action $a$ pairs:
\begin{eqnarray*}
\begin{aligned}
&(\lambda_{1}\rho_{1}(s,a) + \lambda_{2}\rho_{2}(s,a)) \left( 1- \frac{\lambda_{1}\rho_{1}(s,a) + \lambda_{2}\rho_{2}(s,a)}{\lambda_{1}\sum_{a'} \rho_{1}(s,a') + \lambda_{2} \sum_{a'} \rho_{2}(s,a')} \right) \\
&\ge \lambda_{1} \rho_{1}(s,a) \left( 1 - \frac{\rho_{1}(s,a)}{\sum_{a'}\rho_{1}(s,a')} \right) + \lambda_{2} \rho_{2}(s,a) \left( 1 - \frac{\rho_{2}(s,a)}{\sum_{a'}\rho_{2}(s,a')}\right)
\end{aligned}
\end{eqnarray*}
where $\lambda_1 \geq 0$, $\lambda_2 \geq 0$, and $\lambda_1+\lambda_2 = 1$.
For notational simplicity, $\rho_{i}(s,a)$ and $\sum_{a'}\rho_{i}(s,a')$ are replaced with $a_{i}$ and $b_{i}$, respectively.
Then, the right-hand side is
\begin{eqnarray*}
\begin{aligned}
&\sum_{i = 1,2} \lambda_{i} a_{i} \left( 1 - \frac{a_{i}}{b_{i}}\right) = \sum_{i = 1,2} \lambda_{i} a_{i} \left(1 - \frac{ \lambda_{i} a_{i}}{\lambda_{i} b_{i}}\right) \\
&= \left(\sum_{j=1,2} \lambda_{j} b_{j}\right) \sum_{i = 1,2} \left[ \frac{ \lambda_{i} b_{i}}{\left( \sum_{j=1,2} \lambda_{j} b_{j}\right) } \frac{ \lambda_{i} a_{i}}{\lambda_{i} b_{i}} \left(1 - \frac{ \lambda_{i} a_{i}}{\lambda_{i} b_{i}}\right) \right].
\end{aligned}
\end{eqnarray*}
Let $F(x) = x(1 - x)$, which is a concave function. 
Then the above equation can be expressed as follows,
\begin{eqnarray*}
\begin{aligned}
\sum_{i = 1,2} \lambda_{i} a_{i} \left(1 - \frac{a_{i}}{b_{i}}\right) = \left(\sum_{j=1,2} \lambda_{j} b_{j}\right) \sum_{i = 1,2} \left[ \frac{ \lambda_{i} b_{i}}{\left( \sum_{j=1,2} \lambda_{j} b_{j}\right) } F \left(\frac{ \lambda_{i} a_{i}}{\lambda_{i} b_{i}}\right) \right].
\end{aligned}
\end{eqnarray*}
By using the property of concave function $F(x)$\footnote{
$\sum_{i} \mu_i F( x_i) \leq F(\sum_{i} \mu_i x_i)$, for some $(x_i,\dots,x_n)$ and $(\mu_i,\dots,\mu_n)$ such that $\mu_i \ge 0$ and $\sum_{i}\mu_i=1$.
}, we obtain the following inequality:
\begin{eqnarray*}
\begin{aligned}
& \left(\sum_{j=1,2} \lambda_{j} b_{j}\right) \sum_{i = 1,2} \left[ \frac{ \lambda_{i} b_{i}}{\left( \sum_{j=1,2} \lambda_{j} b_{j}\right) } F \left(\frac{ \lambda_{i} a_{i}}{\lambda_{i} b_{i}}\right) \right]\\
&\le \left(\sum_{j=1,2} \lambda_{j} b_{j}\right) F \left(\sum_{i = 1,2} \left[ \frac{ \lambda_{i} b_{i}}{\left( \sum_{j=1,2} \lambda_{j} b_{j}\right) } \frac{ \lambda_{i} a_{i}}{\lambda_{i} b_{i}}\right]\right)
=\left(\sum_{j=1,2} \lambda_{j} b_{j}\right) F \left(\frac{ \sum_{i = 1,2} \lambda_{i} a_{i}}{\sum_{j=1,2} \lambda_{j} b_{j}} \right)\\
&=\left(\sum_{j=1,2} \lambda_{j} b_{j}\right) \frac{ \sum_{i = 1,2} \lambda_{i} a_{i}}{\sum_{j=1,2} \lambda_{j} b_{j}} \left( 1 - \frac{ \sum_{i = 1,2} \lambda_{i} a_{i}}{\sum_{j=1,2} \lambda_{j} b_{j}} \right) =\sum_{i = 1,2} \lambda_{i} a_{i} \left( 1 - \frac{ \sum_{i = 1,2} \lambda_{i} a_{i}}{\sum_{j=1,2} \lambda_{j} b_{j}} \right).
\end{aligned}
\end{eqnarray*}
Finally, we have the following inequality for every state and action pair,
\begin{eqnarray*}
\begin{aligned}
&(\lambda_{1}\rho_{1}(s,a) + \lambda_{2}\rho_{2}(s,a)) \left( 1 - \frac{\lambda_{1}\rho_{1}(s,a) + \lambda_{2}\rho_{2}(s,a)}{\lambda_{1}\sum_{a'} \rho_{1}(s,a') + \lambda_{2} \sum_{a'} \rho_{2}(s,a')}\right) \\
&\ge \lambda_{1} \rho_{1}(s,a) \left(1 -  \frac{\rho_{1}(s,a)}{\sum_{a'}\rho_{1}(s,a')} \right) + \lambda_{2} \rho_{2}(s,a) \left( 1 - \frac{\rho_{2}(s,a)}{\sum_{a'}\rho_{2}(s,a')}\right),
\end{aligned}
\end{eqnarray*}
and, by summing up with respect to $s, a$, we get
\begin{eqnarray*}
\begin{aligned}
\bar{W}(\lambda_{1} \rho_{1} + \lambda_{2} \rho_{2}) \geq \lambda_{1} \bar{W}(\rho_{1}) + \lambda_{2} \bar{W}(\rho_{2}).
\end{aligned}
\end{eqnarray*}
Therefore, $\bar{W}(\rho)$ is a concave function.
\end{proof}

\subsection{Optimality Condition from Karush–Kuhn–Tucker (KKT) conditions}

The following proof explains the optimality condition of the maximum
causal Tsallis entropy problem and also tells us that the optimal
policy distribution has a sparse and multi-modal distribution. 

\begin{proof}[Proof of \textnormal{Theorem} \ref{thm:necessarily_cnd}]
These conditions are derived from the stationary condition of KKT,
where the derivative of $L_{W}$ is equal to zero, 
\begin{equation*}
\frac{\partial L_{W}}{\partial \rho(s,a)} = 0.
\end{equation*}
We first compute the derivative of $\bar{W}$ as follows:
\begin{equation*}
\frac{\partial \bar{W}}{\partial \rho(s,a)} = \frac{1}{2} - \frac{\rho(s,a)}{\sum_{a'} \rho(s,a')} + \frac{1}{2}\sum_{a'}\left(\frac{\rho(s,a')}{\sum_{a'}\rho(s,a')}\right)^2.
\end{equation*}
We also check the derivative of Bellman flow constraints as follows:
\begin{eqnarray*}
\begin{aligned}
\frac{\partial \sum_{s} c_{s} \left(\sum_{a'}\rho(s,a') - d(s) - \gamma \sum_{s',a'} T(s|s',a')\rho(s',a')\right)}{\partial \rho(s'',a'')} = c_{s''}-\gamma \sum_{s} c_{s} T(s|s'',a'').
\end{aligned}
\end{eqnarray*}
Hence, the stationary condition can be obtained as
\begin{eqnarray}\label{eqn:stationary_mcste}
\begin{aligned}
\frac{\partial L_{W}}{\partial \rho(s,a)}=&\alpha\left[-\frac{1}{2} + \frac{\rho(s,a)}{\sum_{a'} \rho(s,a')} - \frac{1}{2}\sum_{a'}\left(\frac{\rho(s,a')}{\sum_{a'}\rho(s,a')}\right)^2\right]-\theta^{\intercal} \phi(s,a) \\
&+ c_{s} - \gamma \sum_{s'}c_{s'}T(s'|s,a)-\lambda_{sa} = 0.
\end{aligned}
\end{eqnarray}
First, let us consider a positive $a \in S(s) = \{a|\rho(s,a) > 0\}$.
From the complementary slackness, the corresponding $\lambda_{sa}$ is zero.
By replacing $\frac{\rho(s,a)}{\sum_a' \rho(s,a')}$ with $\pi_{\rho}(a|s)$ and using the definition of $q_{sa}$,
the following equation is obtained from the stationary condition (\ref{eqn:stationary_mcste}).
\begin{eqnarray} \label{eqn:stationary}
\begin{aligned}
\pi(a|s) - \frac{q_{sa}}{\alpha} = \frac{1}{2} + \frac{1}{2}\sum_{a'}\left(\pi(a'|s)\right)^2 - \frac{c_{s}}{\alpha}.
\end{aligned}
\end{eqnarray}
It can be observed that the right hand side of the equation only depends on the state $s$ and is constant for the action $a$.
In this regards, 
by summing up with respect to the action with positive $\rho(s,a) > 0$,
$c_{s}$ is obtained as follows:
\begin{eqnarray*}
\begin{aligned}
1 - \sum_{a\in S(s)}\frac{q_{sa}}{\alpha} &= K \left( \frac{1}{2} + \frac{1}{2}\sum_{a'}\left(\pi(a'|s)\right)^2 - \frac{c_{s}}{\alpha} \right)\\
\frac{c_{s}}{\alpha} &=  \frac{1}{2} + \frac{1}{2}\sum_{a'}\left(\pi(a'|s)\right)^2 + \frac{\sum_{a\in S(s)}\frac{q_{sa}}{\alpha} - 1}{K},
\end{aligned}
\end{eqnarray*}
where $K$ is the number of actions with positive $\rho(s,a) > 0$.
By plug in $\frac{c_{s}}{\alpha}$ into (\ref{eqn:stationary}), we
obtain a policy as follows:
\begin{eqnarray*}
\begin{aligned}
\pi(a|s) = \frac{q_{sa}}{\alpha} - \left(\frac{\sum_{a\in S(s)}\frac{q_{sa}}{\alpha} - 1}{K}\right)
\end{aligned}
\end{eqnarray*}
Now, we define $\tau(\frac{q_{s}}{\alpha}) \triangleq \frac{\sum_{a\in S(s)}\frac{q_{sa}}{\alpha} - 1}{K}$, and,
interestingly, $\tau$ is the same as the threshold of a sparsemax
distribution \cite{martins2016softmax}. 
Then, we can obtain the optimality condition for the policy
distribution $\pi(a|s)$ as follows:
\begin{eqnarray*}
\begin{aligned}
\forall s,a \;\; \pi(a|s) = \max\left(\frac{q_{sa}}{\alpha} - \tau(s),0\right).
\end{aligned}
\end{eqnarray*}
where $\tau(s)$ indicates $\tau(\frac{q_{s}}{\alpha})$.

The Lagrangian multiplier $c_s$ can be found from $\pi$ as follows:
\begin{eqnarray*}
\begin{aligned}
\frac{c_{s}}{\alpha} &=  \frac{1}{2} + \frac{1}{2}\sum_{a'}\left(\pi(a'|s)\right)^2 + \tau(s)\\
 &=  \frac{1}{2} + \frac{1}{2}\sum_{a' \in S(s)}\left( \frac{q_{sa'}}{\alpha} - \tau(s) \right)^2 + \tau(s)\\
 &=  \frac{1}{2} + \frac{1}{2}\sum_{a' \in S(s)}\left( \frac{q_{sa'}}{\alpha} \right)^2 - \sum_{a' \in S(s)}\frac{q_{sa'}}{\alpha} \tau(s) + \frac{K}{2} \tau(s)^{2} + \tau(s)\\
 &=  \frac{1}{2} + \frac{1}{2}\sum_{a' \in S(s)}\left( \frac{q_{sa'}}{\alpha} \right)^2 - K \frac{ \sum_{a' \in S(s)}\frac{q_{sa'}}{\alpha}-1}{K} \tau(s) + \frac{K}{2} \tau(s)^{2} \\
 &=  \frac{1}{2} + \frac{1}{2}\sum_{a' \in S(s)}\left( \frac{q_{sa'}}{\alpha} \right)^2 - \frac{K}{2} \tau(s)^{2} \\
c_{s} &= \alpha\left[\frac{1}{2}\sum_{a\in S(s)}\left(\left(\frac{q_{sa}}{\alpha}\right)^{2} - \tau \left(\frac{q_{s}}{\alpha}\right)^{2}\right) + \frac{1}{2}\right].
\end{aligned}
\end{eqnarray*}

To summarize, we obtain the optimality condition of (\ref{eqn:maximum_cste_vis}) as follows:
\begin{eqnarray*}
\begin{aligned}
& q_{sa} \triangleq \theta^{\intercal}\phi(s,a) + \gamma
  \sum_{s'}c_{s'}T(s'|s,a), \\
& c_{s} = \alpha\left[\frac{1}{2}\sum_{a\in S(s)}\left(\left(\frac{q_{sa}}{\alpha}\right)^{2} - \tau \left(\frac{q_{s\cdot}}{\alpha}\right)^{2}\right) + \frac{1}{2}\right],\\
& \pi(a|s) = \max\left(\frac{q_{sa}}{\alpha} - \tau
\left(\frac{q_{s\cdot}}{\alpha}\right),0\right).
\end{aligned}
\end{eqnarray*}
\end{proof}

\subsection{Interpretation as Robust Bayes}

In this section, the connection between MCTE estimation and a minimax
game between a decision maker and the nature is explained. 
We prove that the proposed MCTE problem is equivalent to a minimax
game with the Brier score.

\begin{proof}[Proof of \textnormal{Theorem}\ref{thm:maximinprob}]
The objective function can be reformulated as 
\begin{eqnarray*}
\begin{aligned}
&\mathbb{E}_{\tilde{\pi}}\left[\sum_{a'}\frac{1}{2}\left(\mathbbm{1}_{\{a'=a\}} - \pi(a'|s)\right)^2\right]= \mathbb{E}_{\tilde{\pi}}\left[ B(s,a) \right] = \sum_{s,a} \rho_{\tilde{\pi}}(s,a) B(s,a)\\
&=\frac{1}{2}\sum_{s,a}\rho_{\tilde{\pi}}(s,a)\left(1-2\pi(a|s) + \sum_{a'}\pi(a'|s)^{2}\right),
\end{aligned}
\end{eqnarray*}
Hence, the objective function is quadratic with respect to $\pi(a|s)$ and is linear with respect to $\rho_{\tilde{\pi}}(s,a)$.
By using the one-to-one correspondence between $\tilde{\pi}$ and $\rho_{\tilde{\pi}}$,
we change the variable of inner maximization into the state action visitation.
After changing the optimization variable, by using the minimax theorem
\cite{millar1983minimax}, 
the minimization and maximization of the problem
(\ref{eqn:minmaxprob}) are interchangeable as follows: 
\begin{eqnarray*}
\begin{aligned}
\underset{\pi \in \Pi}{\min}\;\underset{\rho_{\tilde{\pi}}\in \mathbf{M}}{\max}\; \mathbb{E}_{\tilde{\pi}}\left[\sum_{a'}\frac{1}{2}\left(\mathbbm{1}_{\{a'=a\}} - \tilde{\pi}(a|s)\right)^2\right] \\
= \underset{\rho_{\tilde{\pi}}\in \mathbf{M}}{\max}\;\underset{\pi \in \Pi}{\min}\; \mathbb{E}_{\tilde{\pi}}\left[\sum_{a'}\frac{1}{2}\left(\mathbbm{1}_{\{a'=a\}} - \tilde{\pi}(a|s)\right)^2\right]
\end{aligned}
\end{eqnarray*}
where sum-to-one, positivity, and Bellman flow constraints are omitted here.
After converting the problem, the optimal solution of inner minimization with respect to $\pi$ is easily computed as $\pi=\tilde{\pi}$ using $\nabla_{\pi(a''|s'') } \mathbb{E}_{\tilde{\pi}}\left[ B(s,a) \right] = 0 $.
After applying $\pi=\tilde{\pi}$ and recovering the variables from $\rho_{\tilde{\pi}}$ to $\tilde{\pi}$, the problem (\ref{eqn:minmaxprob}) is converted into
\begin{eqnarray*}
\begin{aligned}
\underset{\tilde{\pi} \in \Pi}{\max}\;\frac{1}{2}\sum_{s}\rho_{\tilde{\pi}}(s)\left(  1 - \sum_{a}\tilde{\pi}(a|s)^2 \right) = \underset{\tilde{\pi} \in \Pi}{\max}\;W(\tilde{\pi}),
\end{aligned}
\end{eqnarray*}
which equals to the causal Tsallis entropy.
Hence, the problem (\ref{eqn:minmaxprob}) is equivalent to the maximum
causal Tsallis entropy problem. 
\end{proof}

\subsection{Generative Adversarial Setting with Maximum Causal Tsallis Entropy}

\begin{proof}[Proof of \textnormal{Theorem}\ref{thm:gail_setting}]
We first change the variable from $\pi$ to $\rho$ as follows:
\begin{eqnarray}
\begin{aligned}
& \underset{\theta}{\text{max}} \; \underset{\rho}{\text{min}}
& &  -\alpha \bar{W}(\rho)-\theta^{\intercal}\sum_{s,a}\rho(s,a)\phi(s,a) - \theta^{\intercal}\sum_{s,a}\rho_{E}(s,a)\phi(s,a) - \psi(\theta)\\
& \text{subject to}
& & \forall s,a,\;\sum_{s,a}\rho(s,a)\phi(s,a)  = \sum_{s,a}\rho_{E}(s,a)\phi(s,a),\\
&&& \rho(s,a) \geq 0,\;\;\sum_{a}\rho(s,a) = d(s) + \gamma \sum_{s',a'} T(s|s',a')\rho(s',a'),
\end{aligned}
\end{eqnarray}
where $\rho_{E}$ is $\rho_{\pi_{E}}$.
Let 
\begin{eqnarray}
\bar{L}(\rho,\theta) \triangleq  -\alpha \bar{W}(\rho) - \psi(\theta)-\theta^{\intercal}\sum_{s,a}\rho(s,a)\phi(s,a) + \theta^{\intercal}\sum_{s,a}\rho_{E}(s,a)\phi(s,a).
\end{eqnarray}

From Theorem \ref{thm:obj_concave}, $\bar{W}(\rho)$ is a concave function with respect to $\rho$ for a fixed $\theta$.
Hence, $\bar{L}(\rho,\theta)$ is also a concave function with respect to $\rho$ for a fixed $\theta$.
From the convexity of $\psi$, $\bar{L}(\rho,\theta)$ is a convex function with respect to $\theta$ for a fixed $\rho$.
Furthermore, the domain of $\rho$ is compact and convex and the domain of $\theta$ is convex.
Based on this property of $\bar{L}(\rho,\theta)$, we can use minimax duality \cite{millar1983minimax}:
\begin{eqnarray*}
\begin{aligned}
\underset{\theta}{\text{max}} \; \underset{\rho}{\text{min}}\;\; \bar{L}(\rho,\theta) = \underset{\rho}{\text{min}} \; \underset{\theta}{\text{max}}\;\;\bar{L}(\rho,\theta).
\end{aligned}
\end{eqnarray*}
Hence, the maximization and minimization are interchangable.
By using this fact,
we have:
\begin{eqnarray*}
\begin{aligned}
&\underset{\theta}{\text{max}} \; \underset{\rho}{\text{min}}\;\; \bar{L}(\rho,\theta) = \underset{\rho}{\text{min}} \; \underset{\theta}{\text{max}}\;\;\bar{L}(\rho,\theta)\\
& = \underset{\rho}{\text{min}} \;\;  -\alpha \bar{W}(\rho) + \underset{\theta}{\text{max}} \left(- \psi(\theta)+\theta^{\intercal}\sum_{s,a}\left(\rho(s,a) - \rho_{E}(s,a)\right)\phi(s,a) \right)\\
& = \underset{\rho}{\text{min}} \;\;  -\alpha \bar{W}(\rho) + \psi^{*} \left(\sum_{s,a}\left(\rho(s,a) - \rho_{E}(s,a)\right)\phi(s,a) \right) \\
& =  \underset{\pi}{\textnormal{min}}\;\; \psi^{*}\left(\mathbb{E}_{\pi}\left[\phi(s,a)\right] -  \mathbb{E}_{\pi_{E}}\left[\phi(s,a)\right]\right) - \alpha W(\pi)
\end{aligned}
\end{eqnarray*}
\end{proof}

\subsection{Tsallis Entropy of a Mixture of Gaussians}

\begin{proof}[Proof of \textnormal{Theorem}\ref{thm:analytic_mdn}]
The causal Tsallis entropy of a mixture of Gaussian distribution can be obtained as follows:
\begin{eqnarray}
\begin{aligned}
&W(\pi) = \frac{1}{2}\sum_{s} \rho_{\pi}(s) \left( 1 - \int_{\mathcal{A}} \pi(a|s)^2 \mathbf{d}a \right) \\
&= \frac{1}{2}\sum_{s} \rho_{\pi}(s) \left( 1 - \int_{\mathcal{A}} \left(\sum_{i}^{K} w_{i}( s ) \mathcal{N}\left( a ; \mu_i (s) , \Sigma_i (s) \right)\right)^2 \mathbf{d}a \right)\\
&= \frac{1}{2}\sum_{s} \rho_{\pi}(s) \left( 1 - \sum_{i}^{K}\sum_{j}^{K}w_{i}( s )w_{j}( s )\int_{\mathcal{A}} \mathcal{N}\left( a ; \mu_i (s) , \Sigma_i (s) \right)\mathcal{N}\left( a ; \mu_j (s) , \Sigma_j (s) \right) \mathbf{d}a \right)\\
&= \frac{1}{2}\sum_{s} \rho_{\pi}(s) \left( 1 - \sum_{i}^{K}\sum_{j}^{K}w_{i}( s )w_{j}( s )\mathcal{N}\left( \mu_i (s) ; \mu_j (s) , \Sigma_i (s) + \Sigma_j (s) \right)\right)
\end{aligned}
\end{eqnarray}
\end{proof}

\section{Causal Entropy Approximation}

In our implementation of maximum causal Tsallis entropy imitation learning (MCTEIL),
we approximate $W(\pi)$ using sampled trajectories as follows:
\begin{eqnarray} \label{eqn:tse_imple}
\begin{aligned}
W(\pi) &= \mathbbm{E}_{\pi}\left[ \frac{1}{2}\left(1-\pi(a|s)\right) \right] \approx \frac{1}{N}\sum_{i=0}^{N}\sum_{t=0}^{T_{i}}\frac{\gamma^{t}}{2}\left( 1 - \int_{\mathcal{A}} \pi(a|s_{i,t})^2 \mathbf{d}a \right),
\end{aligned}
\end{eqnarray}
where $\{(s_{i,t},a_{i,t})_{t=0}^{T_i}\}_{i=0}^{N}$ are $N$
trajectories and $T_{i}$ is the length of the $i$th trajectory.
Since the integral part of (\ref{eqn:tse_imple}) is analytically computed by Theorem \ref{thm:analytic_mdn},
there is no additional computational cost.
We have also tested the following approximation:
\begin{eqnarray*}
\begin{aligned}
W(\pi) &= \mathbbm{E}_{\pi}\left[ \frac{1}{2}\left(1-\pi(a|s)\right) \right] \approx \frac{1}{N}\sum_{i=0}^{N}\sum_{t=0}^{T_{i}}\frac{\gamma^{t}}{2}\left(1-\pi(a_{i,t}|s_{i,t})\right).
\end{aligned}
\end{eqnarray*}
However, this approximation has performed poorly compared to (\ref{eqn:tse_imple}).

For soft GAIL, $H(\pi)$ is approximated as the sum of discounted likelihoods
\begin{eqnarray*}
\begin{aligned}
H(\pi) &= \mathbbm{E}_{\pi}\left[ -\log\left(\pi(a|s)\right) \right] 	\approx \frac{1}{N}\sum_{i=0}^{N}\sum_{t=0}^{T_{i}}-\gamma^{t}\log\left(\pi(a_{i,t}|s_{i,t})\right).
\end{aligned}
\end{eqnarray*}
Note that the same approximation (\ref{eqn:tse_imple}) of $W(\pi)$ is
not available for $H(\pi)$  
since $-\int_{\mathcal{A}} \pi(a|s) \log\left(\pi(a|s)\right)\mathbf{d}a$ is intractable 
when we model $\pi(a|s)$ as a mixture of Gaussians.

\section{Additional Experimental Results}

In the multi-goal environment, the experimental results with other
hyperparameters are shown in Figure \ref{fig:full_multi_res}. 

\begin{figure}[!h]
\centering
\subfigure[Average Return]{
  \centering
  \includegraphics[width=.4\textwidth]{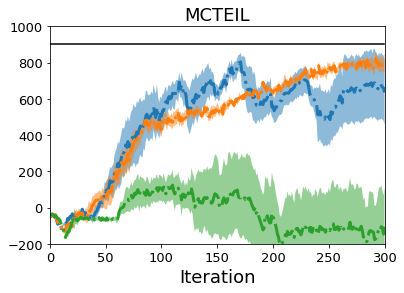}}
  \subfigure[Reachability]{
  \centering
  \includegraphics[width=.5\textwidth]{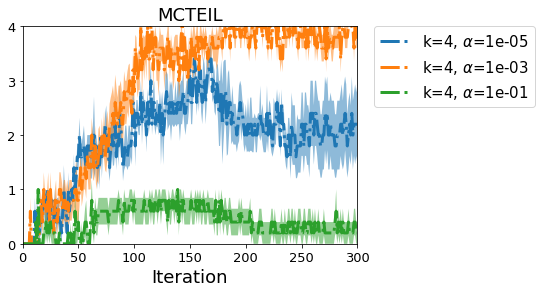}}
\subfigure[Average Return]{
  \centering
  \includegraphics[width=.4\textwidth]{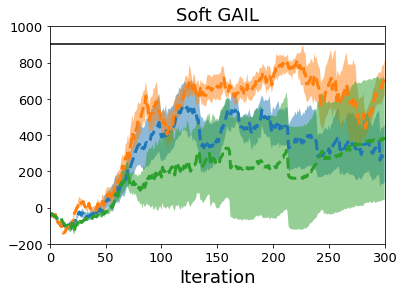}}
  \subfigure[Reachability]{
  \centering
  \includegraphics[width=.5\textwidth]{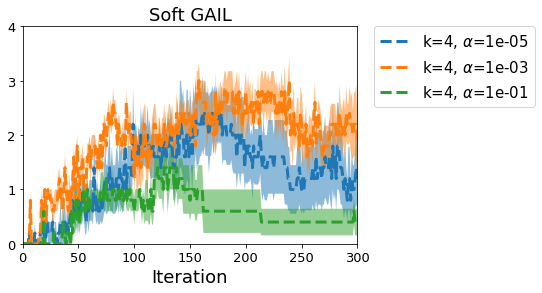}}
\caption{
(a) and (b) show the average return and reachability of MCTEIL, respectively.
(c) and (d) show the average return and reachability of soft GAIL, respectively.
$k$ indicates the number of mixtures and $\alpha$ indicates an entropy regularization coefficient.
}
\label{fig:full_multi_res}
\end{figure}

\bibliographystyle{IEEEtran}
\bibliography{bib_mcteil}

\begin{thebibliography}{10}
\providecommand{\url}[1]{#1}
\csname url@samestyle\endcsname
\providecommand{\newblock}{\relax}
\providecommand{\bibinfo}[2]{#2}
\providecommand{\BIBentrySTDinterwordspacing}{\spaceskip=0pt\relax}
\providecommand{\BIBentryALTinterwordstretchfactor}{4}
\providecommand{\BIBentryALTinterwordspacing}{\spaceskip=\fontdimen2\font plus
\BIBentryALTinterwordstretchfactor\fontdimen3\font minus
  \fontdimen4\font\relax}
\providecommand{\BIBforeignlanguage}[2]{{%
\expandafter\ifx\csname l@#1\endcsname\relax
\typeout{** WARNING: IEEEtran.bst: No hyphenation pattern has been}%
\typeout{** loaded for the language `#1'. Using the pattern for}%
\typeout{** the default language instead.}%
\else
\language=\csname l@#1\endcsname
\fi
#2}}
\providecommand{\BIBdecl}{\relax}
\BIBdecl

\bibitem{ziebart2008maximum}
B.~D. Ziebart, A.~L. Maas, J.~A. Bagnell, and A.~K. Dey, ``Maximum entropy
  inverse reinforcement learning,'' in \emph{Proceedings of the 23rd {AAAI}
  Conference on Artificial Intelligence}, July 2008, pp. 1433--1438.

\bibitem{Haarnoja2017}
T.~Haarnoja, H.~Tang, P.~Abbeel, and S.~Levine, ``Reinforcement learning with
  deep energy-based policies,'' in \emph{Proceedings of the 34th International
  Conference on Machine Learning}, August 2017, pp. 1352--1361.

\bibitem{lee2018sparse}
K.~Lee, S.~Choi, and S.~Oh, ``Sparse {Markov} decision processes with causal
  sparse {Tsallis} entropy regularization for reinforcement learning,''
  \emph{IEEE Robotics and Automation Letters}, vol.~3, no.~3, pp. 1466--1473,
  2018.

\bibitem{Heess2012}
N.~Heess, D.~Silver, and Y.~W. Teh, ``Actor-critic reinforcement learning with
  energy-based policies,'' in \emph{Proceedings of the 10th European Workshop
  on Reinforcement Learning}, June 2012, pp. 43--58.

\bibitem{vamplew2017softmax}
P.~Vamplew, R.~Dazeley, and C.~Foale, ``Softmax exploration strategies for
  multiobjective reinforcement learning,'' \emph{Neurocomputing}, vol. 263, pp.
  74--86, Jun 2017.

\bibitem{martins2016softmax}
A.~F.~T. Martins and R.~F. Astudillo, ``From softmax to sparsemax: {A} sparse
  model of attention and multi-label classification,'' in \emph{Proceedings of
  the 33nd International Conference on Machine Learning}, June 2016, pp.
  1614--1623.

\bibitem{bloem2014infinite}
M.~Bloem and N.~Bambos, ``Infinite time horizon maximum causal entropy inverse
  reinforcement learning,'' in \emph{Proceedings of the 53rd International
  Conference on Decision and Control}, December 2014, pp. 4911--4916.

\bibitem{nachum2018path}
\BIBentryALTinterwordspacing
O.~Nachum, Y.~Chow, and M.~Ghavamzadeh, ``Path consistency learning in tsallis
  entropy regularized mdps,'' 2018. [Online]. Available:
  \url{https://arxiv.org/abs/1802.03501}
\BIBentrySTDinterwordspacing

\bibitem{ho2016generative}
J.~Ho and S.~Ermon, ``Generative adversarial imitation learning,'' in
  \emph{Advances in Neural Information Processing Systems}, December 2016, pp.
  4565--4573.

\bibitem{todorov2012mujoco}
E.~Todorov, T.~Erez, and Y.~Tassa, ``{MuJoCo}: A physics engine for model-based
  control,'' in \emph{Proceedings of the International Conference on
  Intelligent Robots and Systems}, October 2012, pp. 5026--5033.

\bibitem{ziebart2010MPAs}
B.~D. Ziebart, ``Modeling purposeful adaptive behavior with the principle of
  maximum causal entropy,'' Ph.D. dissertation, Carnegie Mellon University,
  Pittsburgh, PA, USA, 2010, aAI3438449.

\bibitem{abbeel2004apprenticeship}
P.~Abbeel and A.~Y. Ng, ``Apprenticeship learning via inverse reinforcement
  learning,'' in \emph{Proceedings of the 21st International Conference of
  Machine Learning}, July 2004.

\bibitem{syed2008game}
U.~Syed and R.~E. Schapire, ``A game-theoretic approach to apprenticeship
  learning,'' in \emph{Advances in neural information processing systems},
  December 2007, pp. 1449--1456.

\bibitem{brier1950verification}
G.~W. Brier, ``Verification of forecasts expressed in terms of probability,''
  \emph{Monthey Weather Review}, vol.~78, no.~1, pp. 1--3, 1950.

\bibitem{syed2008apprenticeship}
U.~Syed, M.~Bowling, and R.~E. Schapire, ``Apprenticeship learning using linear
  programming,'' in \emph{Proceedings of the 25th international conference on
  Machine learning}.\hskip 1em plus 0.5em minus 0.4em\relax ACM, 2008, pp.
  1032--1039.

\bibitem{grunwald2004game}
P.~D. Gr{\"u}nwald and A.~P. Dawid, ``Game theory, maximum entropy, minimum
  discrepancy and robust {Bayesian} decision theory,'' \emph{Annals of
  Statistics}, pp. 1367--1433, 2004.

\bibitem{millar1983minimax}
P.~W. Millar, ``The minimax principle in asymptotic statistical theory,'' in
  \emph{Ecole d’Et{\'e} de Probabilit{\'e}s de Saint-Flour XI—1981}.\hskip
  1em plus 0.5em minus 0.4em\relax Springer, 1983, pp. 75--265.

\bibitem{baram2017end}
N.~Baram, O.~Anschel, I.~Caspi, and S.~Mannor, ``End-to-end differentiable
  adversarial imitation learning,'' in \emph{Proceedings of the 34th
  International Conference on Machine Learning}, August 2017, pp. 390--399.

\bibitem{li2017infogail}
Y.~Li, J.~Song, and S.~Ermon, ``Infogail: Interpretable imitation learning from
  visual demonstrations,'' in \emph{Advances in Neural Information Processing
  Systems}, December 2017, pp. 3815--3825.

\bibitem{hausman2017multi}
K.~Hausman, Y.~Chebotar, S.~Schaal, G.~S. Sukhatme, and J.~J. Lim,
  ``Multi-modal imitation learning from unstructured demonstrations using
  generative adversarial nets,'' in \emph{Advances in Neural Information
  Processing Systems}, December 2017, pp. 1235--1245.

\bibitem{wang2017robust}
Z.~Wang, J.~S. Merel, S.~E. Reed, N.~de~Freitas, G.~Wayne, and N.~Heess,
  ``Robust imitation of diverse behaviors,'' in \emph{Advances in Neural
  Information Processing Systems}, December 2017, pp. 5326--5335.

\bibitem{schulman2015trust}
J.~Schulman, S.~Levine, P.~Abbeel, M.~I. Jordan, and P.~Moritz, ``Trust region
  policy optimization,'' in \emph{Proceedings of the 32nd International
  Conference on Machine Learning}, July 2015, pp. 1889--1897.

\end{thebibliography}
\end{document}